\documentclass{article} % For LaTeX2e
\usepackage{iclr2016_conference,times}
\usepackage{hyperref}
\usepackage{url}
\usepackage{amsmath,amsfonts,amsthm}

\title{Censoring Representations with an Adversary}

\author{
Harrison Edwards \& Amos Storkey  \\
Department of Informatics\\
University of Edinburgh\\
Edinburgh, UK,  EH8 9AB \\
\texttt{H.L.Edwards@sms.ed.ac.uk, A.Storkey@ed.ac.uk} \\
}

\iclrfinalcopy % Uncomment for camera-ready version

\begin{document}

\maketitle

\begin{abstract}
In practice, there are often explicit constraints on what representations or decisions are acceptable in an application of machine learning. For example it may be a legal requirement that a decision must not favour a particular group. Alternatively it can be that that representation of data must not have identifying information. We address these two related issues by learning flexible representations that minimize the capability of an adversarial critic. This adversary is trying to predict the relevant sensitive variable from the representation, and so minimizing the performance of the adversary ensures there is little or no information in the representation about the sensitive variable. We demonstrate this adversarial approach on two problems: making decisions free from discrimination and removing private information from images. We formulate the adversarial model as a minimax problem, and optimize that minimax objective using a stochastic gradient alternate min-max optimizer. We demonstrate the ability to provide discriminant free representations for standard test problems, and compare with previous state of the art methods for fairness, showing statistically significant improvement across most cases. The flexibility of this method is shown via a novel problem: removing annotations from images, from separate training examples of annotated and unannotated images, and with no a priori knowledge of the form of annotation provided to the model. 

\end{abstract}

\section{Introduction}
\label{introduction}
When we apply machine learning techniques many real-world settings, it is not long before we run into the problem of sensitive information. It may be that we want to provide information to a third party, but be sure that third party cannot determine critical sensitive variables. Alternatively it may be that we need to make decisions that do not treat one category differently from another. Two specific cases of this are image anonymization and fairness respectively.

\subsection{Fairness}
As more of life's decisions become automated there is a growing need, for legal and ethical reasons, to have machine learning algorithms that can make \emph{fair} decisions. A decision is fair if it does not depend upon a sensitive variable such as gender, age or race. One naive approach to achieving a fair decision would be to simply remove the sensitive covariate from the model. But information about the sensitive variable can `leak' back into the decisions made if there is any dependence between it and the other variables. In this work we focus on \emph{fair classifiers} where we want to predict a binary variable $Y$ and be fair with respect to a binary sensitive variable $S$. Here, fairness means that the decision is not-dependent on (i.e. marginally independent of) the sensitive variable.

Previous works, such as those listed in Section \ref{fair_classifiers}, have tended to develop specific fair variants of common classifiers to solve this problem. Our approach, called \emph{adversarial learned fair representations} (ALFR), is to learn representations of the data which are concurrently both fair and discriminative for the prediction task. These features can then be used with any classifier. To achieve both fair and discriminative properties, we represent this as a dual objective optimization problem that can be cast as a minimax problem. We maintain the flexibility of both the representation and the test of fairness, by using deep feed-forward neural networks for each part. There is a deep neural network that is used to produce the representation; that representation is then critiqued by a deep neural adversary who tries to predict the sensitive variable from the representation.

In this paper, we introduced the adversarial method ALFR as a minimax problem, describe the optimization process, and evaluate ALFR on two datasets, \emph{Diabetes} and \emph{Adult}. This demonstrates improvement over a related approach \cite{LFR}. We also, as an aside, provide the relationship between the discrimination of a classifier and the H-divergence, as used in domain adaptation. The relationship of these methods to domain adaptation is interesting: the different cases for the sensitive variable can be thought of as different domains. However, we leave the study of this for another paper.

\subsection{Image Anonymization}
There are many notions and problems relating to privacy in the literature. Strict forms of privacy such as that enforced by differential privacy are not always necessary, and more relaxed notions of privacy are beneficial in reducing distortion and making more accurate analysis possible . One particular case of privacy is where certain parts of the data should not be communicated (e.g. someones address or name). However in many settings it is hard to be explicit about exactly what should not be communicated or whether that information is coupled with other measured variables.

In this work we consider the concrete case of removing private information from an image, whilst distorting the image as little as possible. Examples of private information include: licence plates on cars in photos and doctors' annotations on medical images such as X-rays. We suggest a modification of an autoencoder to remove such private information, and validate this idea  by removing surnames from a collection of images of faces. A similar application might be removing logos or watermarks from images. 

The novelty of our approach is that the model does not need to be trained with aligned input/output examples, rather only examples of inputs and (separately) examples of outputs, labelled as such. For example if the task is to remove text from an image, an aligned input/output pair would be an image containing text, and the same image with the text removed. Unaligned data would simply be images labelled as containing no text, and images labelled as containing text. The former sort of data would often be substantially more difficult to obtain than the latter.

Once again we use the same two-neural-network minimax formalism to characterise the problem, and the same stochastic gradient optimization procedure to learn the neural network parameters. The model is applied to the problem of images with and without annotation, to good visual affect. A neural network can no longer distinguish well between annotated and non-annotated images, and the actually annotation itself is obscured.

\section{Related Work}
\label{related_work}
\subsection{Adversarial Learning}
\label{related_adversarial_work}
The idea of adversarial learning is that one has a representation $R$, a dependent variable $S$ and an adversary that tries to predict $S$ from $R$. The adversary then provides an adaptive measure of dependence between $R$ and $S$ which can then be used to learn $R$ so that it minimizes this dependence.

The adversarial approach was, to the best of our knowledge, introduced in \cite{factorial_codes} where it was used to learn a representation of data $R = (R_1, \dots, R_d)$ where each $R_i$ is both binary and independent of the other $R_j$. The experiments in this work were on synthetic data, and were later followed up to learn filters from natural image patches in \cite{semilinear_predictability_minimization} and \cite{semilinear_predictability_minimization2}. They referred to this approach as the \emph{principle of predictability minimization}.

More recently in \cite{gen_adv} the idea of using an adversary to learn a generative model of data was introduced and followed-up by work such as \cite{gen_adv_cond1}, \cite{gen_adv_cond2} and \cite{pyramid_adv}. In this setting the representation $R$ is a mixture of data samples and generated samples, and $S$ is a binary variable indicating whether a given sample is from the data or `fake'. For discussion on using `distinguishability criteria' to learn generative models see \cite{distinguishability_criteria}. 

The inspiration for ALFR comes from using adversarial learning to do domain-adaptation in \cite{domain_adapt_adv}. In this setting $S$ is a variable indicating the domain, and the idea is to learn a representation $R$ in which the domains are indistinguishable, motivated by bounds in \cite{theory_different_domains} relating performance on the target domain (the new domain to which we want to adapt) to the dissimilarity of the target and source domains.

\subsection{Fair Classifiers}
\label{fair_classifiers}
Several works have proposed variants of classifiers that enforce fairness. These include discrimination-free naive-Bayes (\cite{fair_naive_bayes}), a regularized version of logistic regression (\cite{fair_reg_log_reg}), and a more recent approach \cite{fairness_constraint}, where the authors introduce constraints into the objective functions of logistic regression, hinge-loss classifiers and support vector machines.

Another approach is \emph{data massaging}, whereby the labels of the training data are changed so that the training data is fair (\cite{fair_relabelling}), this is similar to the resampling methods often used to tackle class-imbalance problems.

An approach more in the spirit of ALFR is \emph{learned fair representations} (LFR) (\cite{LFR}). In that paper, the authors learn a representation of the data that is a probability distribution over clusters --- a form of `fair clustering' --- where learning the cluster of a datapoint tells one nothing about the sensitive variable $S$. The clustering is learned to be fair and also discriminative for the prediction task.

Other than LFR, the previous work has focused on enforcing fair decisions, whereas LFR and ALFR aim to get fairness as a side-effect of fair representations. The advantage of the latter approach is that the representations can potentially be reused for different tasks, and there is the possibility of a separation of concerns whereby one party is responsible for making the representations fair, and another party is responsible for making the best predictive model. In contrast to LFR, our approach is more flexible in terms of the kinds of representations it can learn, whereas LFR learns essentially categorical representations. In addition our approach means that the representations can be used with any classifier.

Concurrent with this work is the preprint `The Variational Fair Autoencoder' (\cite{variational_fair_autoencoder}). There are two main differences between this work and theirs. First they use a variational autoencoder that factorizes the latent variables $Z$ and sensitive variable $S$. This aspect is complementary and could be incorporated into the adversarial framework. Secondly they use a Maximum Mean Discrepancy (MMD) penalty to reduce dependence of the representation on $S$, this is a kernel-based alternative to using an adversary. It is not yet clear in what circumstances we should prefer MMD over an adversary, we hope that future work will address this question.

\subsection{Removing Private Information}
The problem of removing information from images in a learned manner has not been tackled, to the best of our knowledge, in the machine learning community. However there has been work on detecting when private information has been erroneously left in an image, such as in \cite{screenavoider} and \cite{neural_net_privacy}

\section{Formalism: Fairness and Discrimination}
\label{fairness_and_disc}

We consider a binary classification task where $\mathcal{X} \subset \Real^n$ is the input space, $\mathcal{Y}=\{0,1\}$ is the label set and $\mathcal{S}=\{0,1 \}$ is a protected variable label set. 
We are provided with \emph{i.i.d} examples $\{(x_i, y_i, s_i) \}^{m}_{i=1}$ drawn from some joint distribution $P$ over corresponding random variables $(X, Y,S)$. The goal of the learning algorithm is to build a classifier $\eta: \mathcal{X} \to \mathcal{Y}$ that has high accuracy whilst maintaining the property of \emph{statistical parity} or \emph{fairness}, that is 
\begin{equation}
P(\eta(X)=1 | S =1) = P(\eta(X)=1 | S= 0).
\end{equation}

A key statistic we will use to measure statistical parity is the \emph{discrimination} defined
\begin{equation}
y_{disc} = \left | \frac{\suml{i: s_i=0}{} \eta(x_i)}{N_0} -  \frac{\suml{i: s_i=1}{} \eta(x_i)}{N_1} \right |,
\end{equation}
where $N_0,N_1$ are the number of data items where $s_i$ equal to $0$ and $1$, respectively. We measure the success of our classifier using the empirical accuracy $y_{acc}$.

Following \cite{LFR} we aim to optimize the difference between discrimination and classification accuracy
\begin{equation}
y_{t,delta}= y_{acc} - t \cdot y_{disc},
\end{equation}
where $t \ge 0$, called the \emph{delta}. In \cite{LFR} they consider the specific trade off where $t=1$, whereas we will evaluate our models across a range of different values for $t$.

\section{Censored Representations}
\label{censored_representations}
In this section we show how a modification of an autoencoder can be used to learn a representation that obscures/removes a sensitive variable. In the general case we have an input variable $X$, a target variable $Y$ and a binary sensitive variable $S$. The objective is to learn a representation $R = \text{Enc}(X)$ that preserves information about $X$, is useful for predicting $Y$ and is approximately independent of $S$. The loss we will use is of the form
\begin{equation}
L = \alpha C(X, R) + \beta D(S,R) + \gamma E(Y, R) 
\end{equation}
where $C$ is the cost of reconstructing $X$ from $R$, $D$ is a measure of dependence between $R$ and $S$ and $E$ is the error in predicing $Y$ from $R$. The scalars $\alpha, \beta, \gamma \ge 0$ are hyperparameters controlling the weighting of the competing objectives. If we don't have a specific prediction task, as in the case where we just want to remove private information from an image, then we don't need the $\gamma$ term. On the other hand, if we are not interested in reusing the representation for different predictive tasks then we may set $\alpha = 0$. We may also want to have $\alpha >0$ in the case where we want to learn a transformation of the data $X \to \hat{X}$ that is fair and preserves the semantics of the representation, for example in certain regulated areas the interpretability of the model is paramount. Notice also that only the final term depends-upon $Y$ and so there is an opportunity to train in a semi-supervized fashion if the labels are not available for all of the data.

\subsection{Quantifying Dependence}
%Can this be handled via information theoretic rather than VC arguments?
We begin with quantifying the dependence $D(S,R)$ between the representation $R$ and the sensitive variable $S$. Since $S$ is binary, we can think of this as measuring the difference between two conditional distributions:  $R_0 \sim R | (S= 0)$ and $R_1 \sim R | (S =1)$. We will do this by training a classifier to tell them apart, called the \emph{adversary}. Interested readers may wish to know that this measure of dependence is related to the notion of an $\mathcal{H}$-divergence, as described in Appendix \ref{App:AppendixA}.

In particular if the adversary network $\text{Adv}: \mathcal{R} \to [0,1]$ trained to discriminate between $R_0$ and $R_1$ has parameters $\phi$, and the encoder $\text{Enc}: \mathcal{X} \to \mathcal{R}$ has parameters $\theta$, then we can define
\begin{equation}
\label{dep_cost}
D_{\theta, \phi}(R,S) = \Expt_{X,S} S \cdot \log \left ( \text{Adv}(R) \right )+ (1-S) \cdot \log \left ( 1-\text{Adv}(R) \right ),
\end{equation}

that is the \emph{negative} of the standard log-loss for a binary classifier. The adversary's parameters $\phi$ should be chosen to maximize $D$ (and hence to approximately realize the empirical $\mathcal{H}$-divergence $\hat{d}_{\mathcal{H}}(A,B)$), whilst the representation parameters $\theta$ should be chosen to minimize $D$ (and hence to approximately minimize $\hat{d}_{\mathcal{H}}(A,B)$), so we have a minimax problem:
\begin{equation}
\min _{\theta} \max_{\phi} D_{\theta, \phi}(R,S).
\end{equation}
Of course, this admits a trivial solution by learning a constant representation $R$, and so we must introduce constraints to learn anything interesting.

\subsection{Quantifying C(X,R)}
We quantify the information retained in $R$ about $X$ by the ability of a decoder $\text{Dec}: \mathcal{R} \to \mathcal{X}$ to reconstruct $X$ from $R$, in particular we use the expected squared error
\begin{equation}
\label{recon_cost}
C_{\theta}(X,R) = \Expt_X  \norm{X - \text{Dec}(R)}^2_2,
\end{equation}
where we extend $\theta$ to include the parameters in the decoder Dec.

\subsection{Quantifying E(Y,R)}
We quantify how discriminative $R$ is for the prediction task using the log-loss of a classifier or predictor network $\text{Pred}: \mathcal{R} \to [0,1]$, trained to predict $Y$ from $R$:
\begin{equation}
\label{pred_cost}
E_{\theta}(R,S) = - \Expt_{X,Y} Y \cdot \log \left ( \text{Pred}(R)  \right ) + (1-Y) \cdot \log \left ( 1- \text{Pred}(R) \right ),
\end{equation}
where we again extend $\theta$ to encompass the parameters of the predictor network Pred.

\subsection{Optimization}
There are four elements in the model: the encoder, the decoder, the predictor and the adversary. Each is implemented using a feed-forward neural network, the precize architectures are detailed in Section \ref{experimental_results}. The costs for these elements, given in equations \ref{pred_cost}, \ref{dep_cost} and \ref{recon_cost}, are joined together to give the joint loss

\begin{equation}
 L(\theta,\phi) = \alpha C_{\theta} (X,R) + \beta D_{\theta, \phi}(R,S) +  \gamma E_{\theta}(Y,R)
\end{equation}
This enables the problem of learning censored representation to be cast as the minimax problem
\begin{equation}
\min_{\theta} \max_{\phi} L(\theta, \phi).
\end{equation}
Expecting to be able to train the adversary in the inner loop to a global optimum is unrealistic. Instead, as described in \cite{gen_adv} we use a heuristic: a variant on stochastic gradient descent where for each minibatch we decide whether to take a gradient step with respect to the actors parameters $\theta$ or a negative gradient step with respect to the adversary's parameters $\phi$. In \cite{gen_adv} they consider simple strict alternation between updating the adversary and actor, we find this to be a useful default. We give detailed pseudo-code for strict alternation in Algorithm \ref{training_algorithm}. Note that the gradient steps in Algorithm \ref{training_algorithm} can easily be replaced with a more powerful optimizer such as the \emph{Adam} algorithm (\cite{Adam}). This method is a heuristic in the sense that we do not provide formal guarantees of convergence, but we do know that the solution to the minimax problem is a fixed point of the process. There are a large number of papers using this method (such as those mentioned in \ref{related_adversarial_work}) and getting good results, so there is considerable empirical evidence in its favour.

The key issue in this process is that if the adversary is too competent then the gradients will be weak for the actor, whereas if the adversary is too incompetent then the gradients will be uninformative for the actor.  We have also considered not updating the adversary if, for instance, its accuracy in predicting $S$ is over a threshold, say $90\%$, and not updating the generator if the accuracy is below a threshold, say $60\%$. We find that this sometimes improves results, but more investigation is needed.
\begin{algorithm}                        % and a label for \ref{} commands later 
\caption{Strictly alternating gradient steps.}
\label{training_algorithm}
\begin{algorithmic}                    % enter the algorithmic environment
    \State $\theta, \phi \leftarrow$ initialize network parameters
    \State $U \leftarrow \text{True}$ \Comment{Boolean indicating whether to update parameters $\theta$ or $\phi$.}
    \Repeat
    \State $X,Y,S \leftarrow$ random mini-batch from dataset
    \State $R \leftarrow \text{Enc}(X)$
    \State $\hat{X} \leftarrow \text{Dec}(R)$
    \State $\hat{Y} \leftarrow \text{Pred}(R)$
    \State $\hat{S} \leftarrow \text{Adv}(R)$
    \State $C \leftarrow \frac{1}{|X|} \suml{x, \hat{x} \in X,\hat{X}}{} \norm{x - \hat{x}}^2_2$ \Comment{Reconstruction loss for the autoencder.}
    \State $E \leftarrow \frac{-1}{|X|} \suml{y, \hat{y} \in  Y,\hat{Y}}{} y \cdot \log (\hat{y}) + (1-y) \cdot \log (1-\hat{y})$ \Comment{Log-loss for the predictor.}
    \State $D \leftarrow \frac{1}{|X|} \suml{s, \hat{s} \in S, \hat{S}}{} s \cdot \log (\hat{s}) + (1-s) \cdot \log (1-\hat{s}) $ \Comment{Negative log-loss for the adversary.}
    \State $L \leftarrow \alpha  C + \beta  D + \gamma  E$ \Comment{Joint loss.}

    \If{$U$} \Comment{Updating adversary's parameters.}
            \State $\theta \leftarrow \phi  + \alpha \nabla_{\phi} L $
    \Else \Comment{Updating autoencoder's parameters.}
    \State $\phi \leftarrow \theta  -  \alpha \nabla_{\theta} L $
    \EndIf
    \State $U \leftarrow \text{not } U$
    \Until{Deadline}

\end{algorithmic}
\end{algorithm}

\subsection{Adversarial Learned Fair Representations (ALFR)}
We apply the general setup described in Section \ref{censored_representations} to the case of learning fair classifiers as described in Section \ref{fairness_and_disc}. In this case the sensitive variable $S$ would correspond to some category like gender or race, and the target variable $Y$ would be the attribute we wish to predict using these fair representations. 

In this specific case it is worth pointing out that the discrimination $y_{disc}$ of the classifier $\eta$ is bounded above by the empirical $\mathcal{H}$-divergence, as shown in Lemma \ref{disc_divergence_lemma} in Appendix \ref{App:AppendixA}, that is
\begin{equation}
y_{disc} \le \frac{1}{2} \hat{d}_{\mathcal{H}}(\hat{R}_0,\hat{R}_1),
\end{equation}
where $\mathcal{H}$ is a symmetric hypothesis class on $R$ including $\eta$ and $\hat{R}_0, \hat{R}_1$ are the empirical samples given to us. This is because the discrimination between $R_0$ and $R_1$ given by the best hypothesis must be at least as good as some particular hypothesis $\eta$, assuming $\eta \in \mathcal{H}$. So minimizing the divergence minimizes the discrimination.

\subsection{Anonymizing Images}
We apply the the idea of censoring representations to the application of removing text from images. In this problem the input $X$ is an image and the sensitive variable $S$ describes whether or not the image contains private information (text). Here there is no prediction task and so there is no variable $Y$. In contrast to ALFR, we are not interested in learning a hidden representation, but the reconstructed image $\hat{X}$, so in this case we have $R = \hat{X}$. In order to evaluate the model, we need to have a small amount of validation/test data where we have pairs of images with and without the text. This is used to choose hyperparameters, but the model itself never gets to train on example input/output pairs.

\section{Experimental Results}
\label{experimental_results}
We used the \emph{Adam} algorithm with the default parameters for all our optimizations. The experiments were implemented using \emph{theano} (\cite{theano1}, \cite{theano2}) and the python library \emph{lasagne}.

\subsection{Fairness}

\subsubsection{Datasets}
We used two datasets from the UCI repository \cite{UCI_rep} to demonstrate the efficacy of ALFR.

The \emph{Adult} dataset consists of census data and the task is to predict whether a person makes over $50K$ dollars per year. The sensitive attribute we chose was \emph{Gender}. The data has $45,222$ instances and $102$ attributes. We used $35$ thousand instances for the training set and approximately $5$ thousand instances each for the validation and test sets.

The \emph{Diabetes} dataset consists of hospital data from the US and the task is to predict whether a patient will be readmitted to hospital. The sensitive attribute we chose was \emph{Race}, we changed this to a binary variable by creating a new attribute \emph{isCaucasian}. The data has around $100$ thousand instances and $235$ attributes. We used $80$ thousand instances for the training set and approximately $10$ thousand instances each for the validation and test sets. 

\subsubsection{Protocols and Architecture}
To compare ALFR with LFR we split each dataset into training, validation and test sets randomly, we then run $100$ experiments per model with different hyperparameters. Then for each value of $t$ considered we selected the model maximizing $y_{t, disc}$ on the validation data. This process was repeated $5$ times on different data splits. Since each model sees the same data split, the observations are paired and so we get $5$ observations of the difference in performance for each value of $t$. 

In evaluating the results, we want to evaluate the approaches for different possible tradeoffs between accuracy and discrimination, and so we compare with a range of values of $t \in [0,3]$ and explore hyperparameters using a random search (the $100$ settings over hyperparameters are drawn from a product of simple distributions over each hyperparameter). We use random search, as opposed to a sequential, performance driven search, so that we are able to compare the models across a range of values of $t$. When using these methods in practice one should select the hyperparameters, using Bayesian optimization or a similar approach, to maximize the specific tradeoff one cares about.

We now give details of the priors over hyperparameters used for the random search. The autoencoder in ALFR had $U\left( \{ 1, \dots, 3 \} \right )$ encoding/decoding layers with $U\left( \{ 1, \dots, 100 \} \right )$ hidden units, with all hidden layers having the same number of hidden units. Each encoding/decoding unit used the ReLU (\cite{ReLU}) activation. The critic also had $U\left( \{ 1, \dots, 3 \} \right )$ hidden layers with ReLU activations. The predictor network was simply a logistic regressor on top of the central hidden layer of the autoencoder. In the LFR model the model had $U\left( \{ 5, \dots, 50 \} \right )$ clusters. In both models the reconstruction error weighting parameter $\alpha$ was fixed at $0.05$. For both models we used had $\beta \sim U[0,50]$ and $\gamma \sim U[0,10]$.

\subsubsection{Results}
We found the LFR model was sensitive to hyperparameters/initialization and could often stick during training, but given sufficient experiments we were able to obtain good results. In Figure \ref{adult_results} and  Figure \ref{diabetes_results} we see the  results of applying both LFR and ALFR on the \emph{Adult} and \emph{Diabetes} data respectively. In both cases we see that the ALFR model is able to obtain significantly better results across most of the range of possible tradeoffs between accuracy and discrimination. The step changes in these figures correspond to places where the change in tradeoff results in a different form of model. This also clarifies that the results are often not very sensitive to this tradeoff parameter.

\begin{figure}[h]

\begin{center}
\includegraphics[width = 1\linewidth]{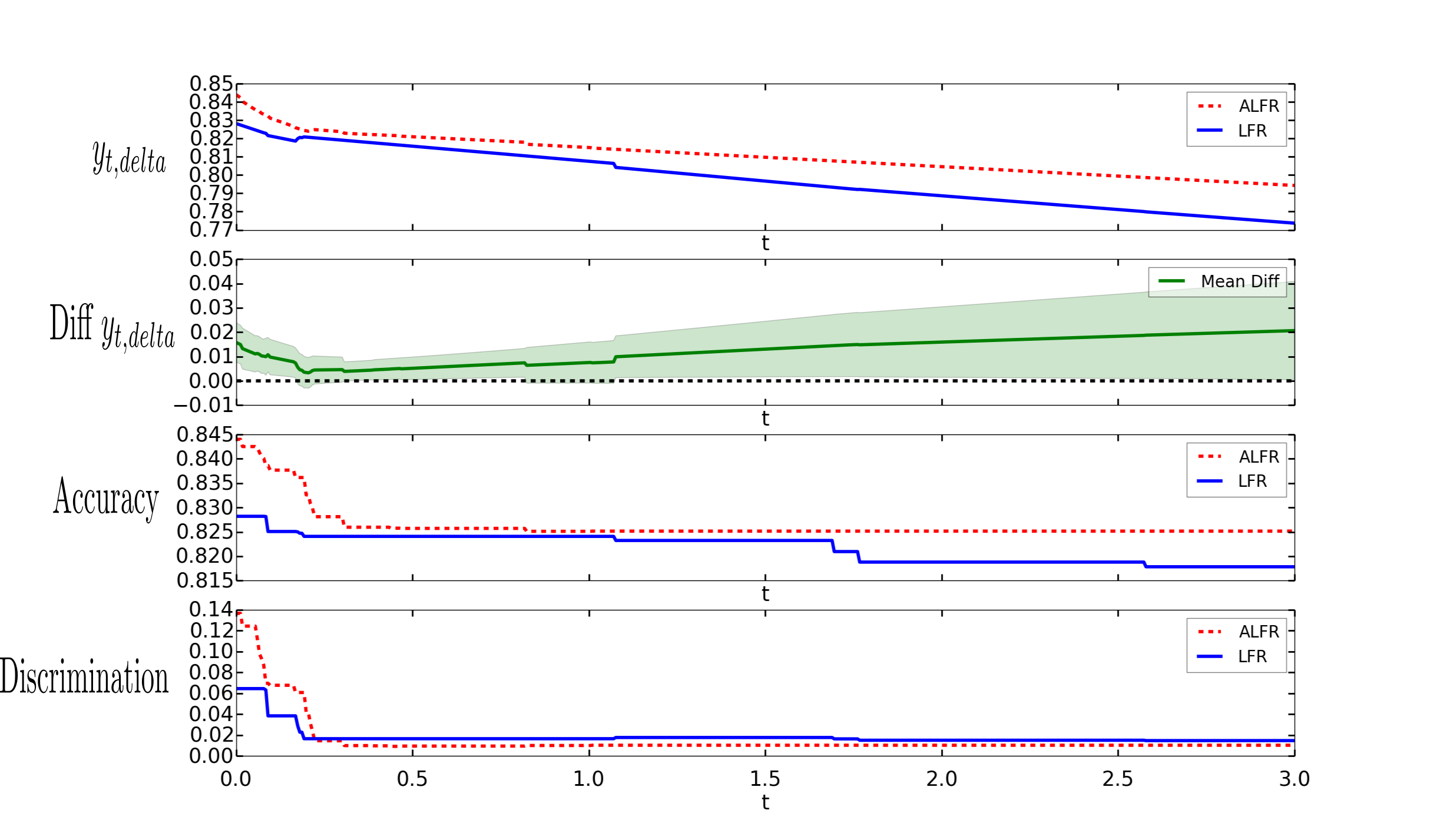}
\end{center}
\caption{Results on \emph{Adult} dataset. For each value of $t$ the model maximizing $y_{t, delta}$ on the validation set is selected, the plots show the performance of the selected models on the test data. Top row is $y_{t,delta}$. Second row down shows the mean paired difference between the $y_{t, delta}$ for the ALFR model and LFR model, where positive values favour ALFR. We also give a 95\% CI around the mean. Third row  down is $y_{acc}$. Bottom row is $y_{disc}$. We see from the top row that the ALFR model has better $y_{t,delta}$ for every setting of $t$ considered. Moreover the difference is significant for most values of $t$ (that is, the CI does not include zero). \label{adult_results}}
\end{figure}

\begin{figure}[h]

\begin{center}
\includegraphics[width = 1\linewidth]{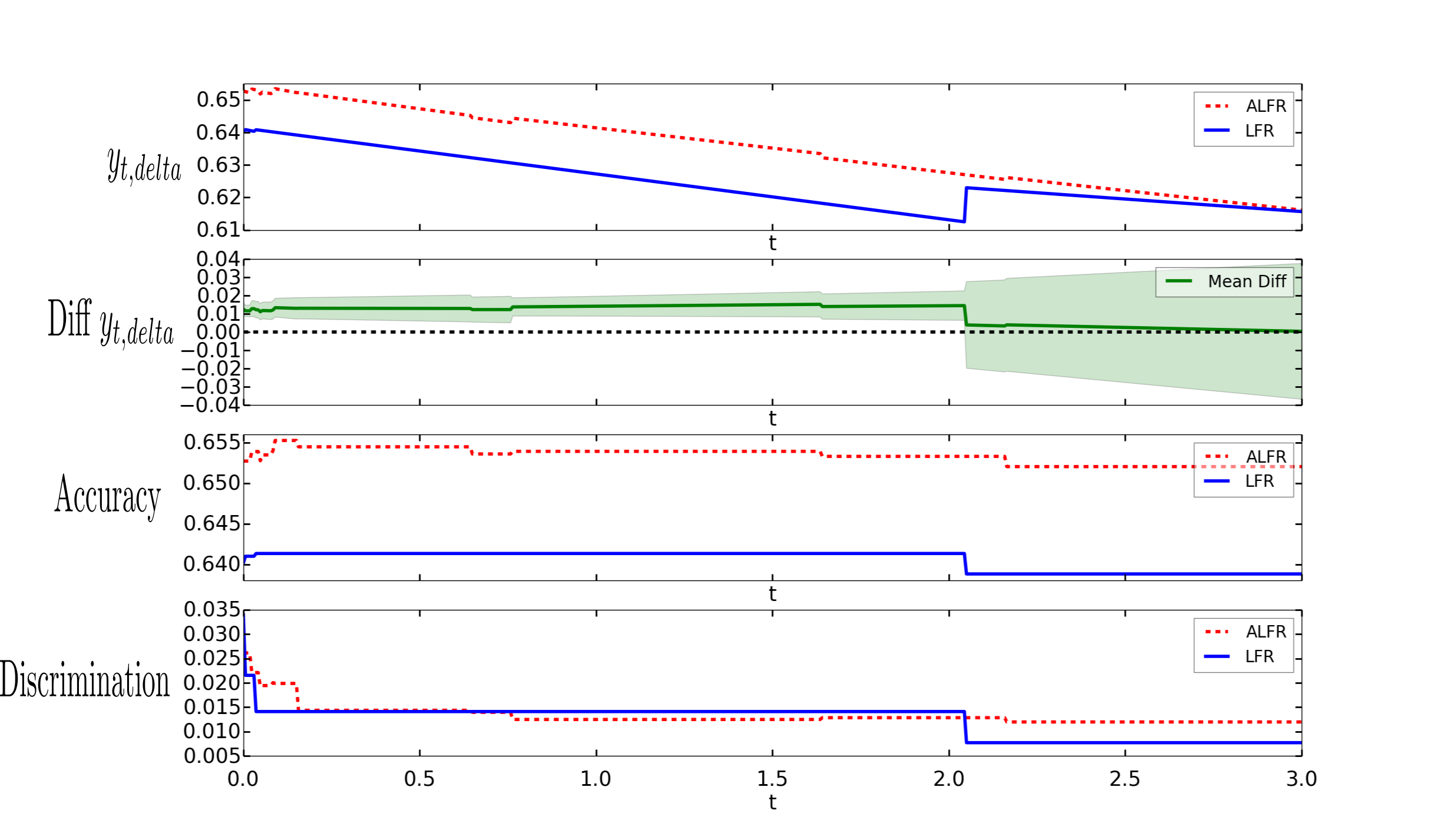}
\end{center}
\caption{Results on \emph{Diabetes} dataset. For each value of $t$ the model maximizing $y_{t, detla}$ on the validation set is selected, the plots show the performance of the selected models on the test data. Top row is $y_{t,delta}$. Second row down shows the mean paired difference between the $y_{t, delta}$ for the ALFR model and LFR model, where positive values favour ALFR. We also give a 95\% CI around the mean. Third row  down is $y_{acc}$. Bottom row is $y_{disc}$.We see from the top row that the ALFR model has better $y_{t,delta}$ for every setting of $t$ considered. We also see that the difference is significant for $ \approx t \le 2$ (that is, the CI does not include zero).\label{diabetes_results} }
\end{figure}

We were also interested in the effect of the hyperparameters $\alpha$, $\beta$, $\gamma$ on the discrimination of the ALFR model. In particular we consider the ratio $\frac{\beta}{\beta + \gamma}$, measuring the relative importance of the dependence term $D(X,R)$ over the prediction error term $E(Y,R)$ in the cost. In Figure \ref{hyperparams} we see that the larger $\beta$ is relative to $\gamma$, the lower the discrimination is (up to a point), which matches expectations.

\begin{figure}[h]

\begin{center}
\includegraphics[width = 1\linewidth]{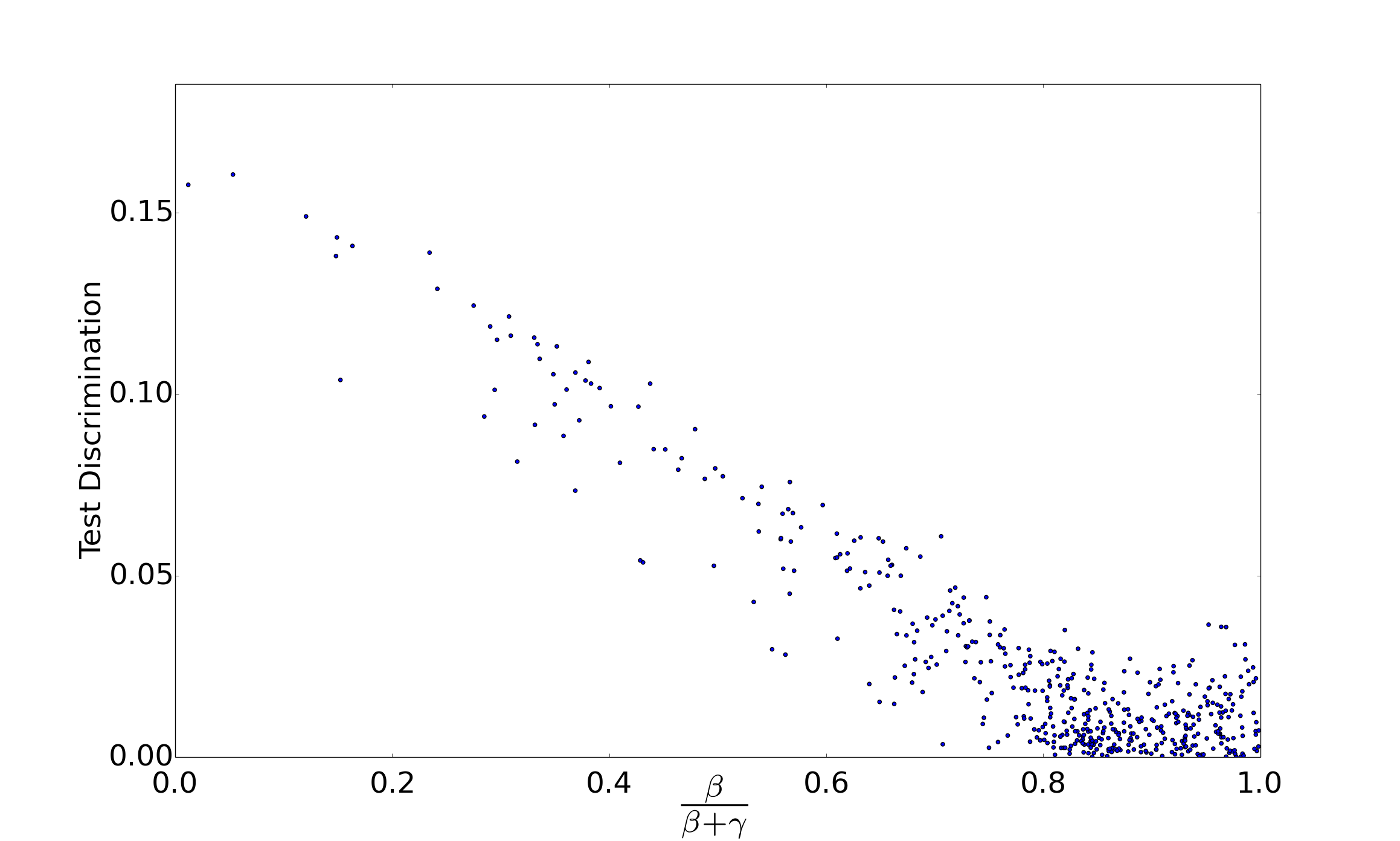}
\end{center}
\caption{A scatter plot of $\frac{\beta}{\beta+\gamma}$ versus the test discrimination of the ALFR model on the \emph{Adult} dataset. We see an approximately linear relation up to a certain point, after which further relative increase in $\beta$ has little effect on the discrimination. \label{hyperparams}}
\end{figure}

\subsection{Image Anonymization}
\subsubsection{Datasets}
We used the `10k US Adult Faces' dataset \cite{10kfaces} consisting of 10,168 natural face photographs. We preprocessed the data by rescaling each to $100 \times 80$ pixels and converting to grayscale. In addition, for half the training images we added private information, in the form of text overlayed at a random location in the image. We used $200$ photographs as a validation set, and $49$ for the test set, on both the validation and test sets we have the image both with and without the text for evaluation. In this problem the sensitive variable $S$ indicates whether or not the image contains text.

\subsubsection{Protocols and Architecture}
The global image autoencoder model we used was an expert patch based model, using $5\times5$ patches. The expert model consists of two parts, a patch classifier and a neural patch autoencoder. The image reconstruction is formed by simply copying the patch if the patch classifier predicts there to be no text-label in the patch. Otherwize it uses the patch autoencoder to construct the patch. The patch classifier is pretrained with weak-supervision using image level labels for each patch.  The decision boundary for the patch adversary is a hyperparameter, optimized to use the autoencoder if the probability of a text-label is $>0.7$ in this setting. Altogether the whole image was reconstructed from the input image using the expert patch based model for each patch of the image.

Along with the image autoencoder is the full image adversary, which takes the whole reconstructed image and classifies whether the image has a text-label present or not.

The patch classifier had a single hidden layer with 200 ReLU units. The autoencoder also had a single hidden layer with $2000$ ReLU units. The adversary had two convolutional layers each with $10$ filters of size $3 \times 3$ interspersed with two max-pooling layers with pooling size $4$, this was followed by a dense layer with $1000$ hidden units with ReLU activations. Other hyperparameters were the $\alpha$ and $\beta$ weights, which were chosen to be $\alpha =1$ and $\beta =10$. 

All hyperparameters were chosen by evaluation on the validation data. The validation data consists of input/output pairs $(x_{in}, x_{out})$ where $x_{in}$ contains text, and $x_{out}$ is the corresponding image without text. The performance was measured using the mean-square error betweeen the autoencoder's output $\hat{x}$ given $x_{in}$ and the target $x_{out}$.

\subsubsection{Results}
Sample images produced by the trained model on the test data are shown in Figure \ref{faces}. The model learns to imagine what the image would have looked like had there been no text, despite having never been given example input/output pairs. The produced images are quite plausible, but artifacts become apparent when zooming in. We believe that the quality of the images could be substantially improved through the use of a convolutional autoencoder on a larger dataset, since this could take into account a wider context than a patch when removing the text. 

In this application the flexibility of the adversarial framework as compared with a clustering approach becomes apparent, it would be extremely difficult to get an LFR-style approach to work for images, since it must reconstruct images as a convex combination of template images.
\begin{figure}[h]

\begin{center}
\includegraphics[width = 1\linewidth]{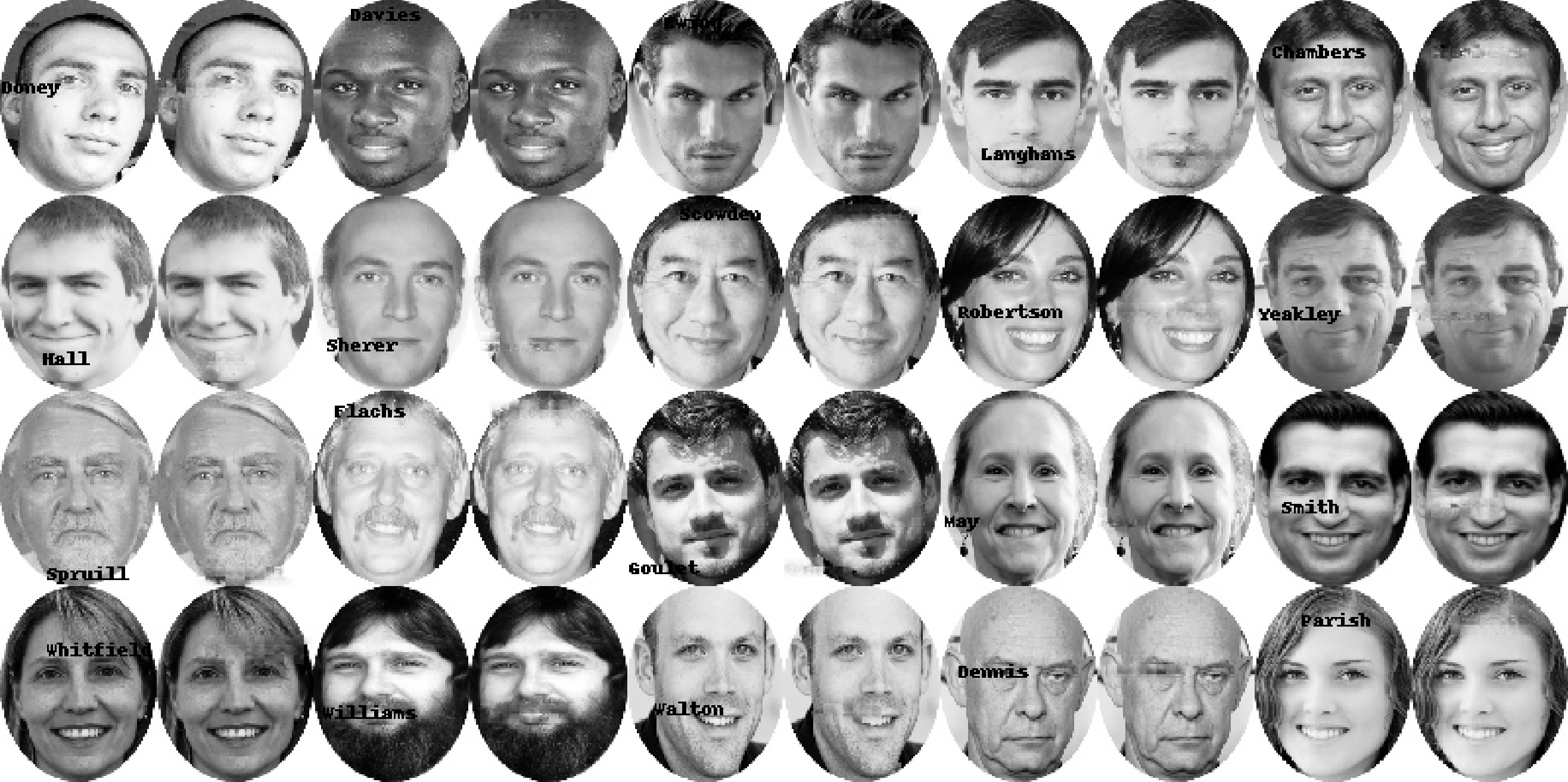}
\end{center}
\caption{Image anonymization results on the test set. In each pair of faces the left image is the input to the autoencoder, and the right image is the censored output. \label{faces} }
\end{figure}

\section{Conclusions and Future Work}
We have shown how the adversarial approach can be adapted to the task of removing sensitive information from representations. Our model ALFR improves upon a related approach whilst at the same time being more flexible. We demonstrate this flexibility by showing how the same setup can be used to remove text from an image, with encouraging results.

As has been noted before it is difficult to train adversarial models owing to the unstable dynamic between actor and adversary. There is work to be done in the future in developing theory, or at least heuristics, for improving the stability of the training process.

Following up from the application on images we would like to investigate the more challenging problem of removing more pervasive information from an image, for instance removing gender from a face.

Another interesting problem to investigate would be obscuring text in images where we only have negative examples of images with text. With no further assumptions our method would not be applicable, but if we assume that we have some information about the text then we can gain some traction. For example if the text is the name of the person in the image, and we know the name of the person in the image, then the adversary could be trained to predict a bag-of-characters of the name, whereas the autoencoder would be trained to make this task difficult for the adversary. The result should be a blurring, rather than removal of the text. One issue one would face in this approach would be a lack of any ground-truth examples for validation, since there are many ways to obscure text.

\subsubsection*{Acknowledgments}

This work was supported in part by the EPSRC Centre for Doctoral Training in Data Science, funded by the UK Engineering and Physical Sciences Research Council (grant EP/L016427/1) and the University of Edinburgh.

\bibliography{iclr2016_conference}

\begin{thebibliography}{26}
\providecommand{\natexlab}[1]{#1}
\providecommand{\url}[1]{\texttt{#1}}
\expandafter\ifx\csname urlstyle\endcsname\relax
  \providecommand{\doi}[1]{doi: #1}\else
  \providecommand{\doi}{doi: \begingroup \urlstyle{rm}\Url}\fi

\bibitem[Bainbridge et~al.(2013)Bainbridge, Isola, and Oliva]{10kfaces}
Bainbridge, Wilma~A, Isola, Phillip, and Oliva, Aude.
\newblock The intrinsic memorability of face photographs.
\newblock \emph{Journal of Experimental Psychology: General}, 142\penalty0
  (4):\penalty0 1323, 2013.

\bibitem[{Bastien} et~al.(2012){Bastien}, {Lamblin}, Pascanu, Bergstra,
  Goodfellow, {Bergeron}, {Bouchard}, and {Bengio}]{theano2}
{Bastien}, Fr{\'{e}}d{\'{e}}ric, {Lamblin}, {Pascal}, Pascanu, Razvan,
  Bergstra, James, Goodfellow, Ian~J., {Bergeron}, {Arnaud}, {Bouchard},
  {Nicolas}, and {Bengio}, {Yoshua}.
\newblock Theano: New features and speed improvements.
\newblock Deep Learning and Unsupervised Feature Learning {NIPS} 2012 Workshop,
  2012.

\bibitem[Ben-David et~al.(2007)Ben-David, Blitzer, Crammer, Pereira,
  et~al.]{rep_dom}
Ben-David, Shai, Blitzer, John, Crammer, Koby, Pereira, Fernando, et~al.
\newblock Analysis of representations for domain adaptation.
\newblock \emph{Advances in neural information processing systems},
  19:\penalty0 137, 2007.

\bibitem[Ben-David et~al.(2010)Ben-David, Blitzer, Crammer, Kulesza, Pereira,
  and Vaughan]{theory_different_domains}
Ben-David, Shai, Blitzer, John, Crammer, Koby, Kulesza, Alex, Pereira,
  Fernando, and Vaughan, JenniferWortman.
\newblock A theory of learning from different domains.
\newblock \emph{Machine Learning}, 79\penalty0 (1-2):\penalty0 151--175, 2010.
\newblock ISSN 0885-6125.
\newblock \doi{10.1007/s10994-009-5152-4}.
\newblock URL \url{http://dx.doi.org/10.1007/s10994-009-5152-4}.

\bibitem[Bergstra et~al.(2010)Bergstra, Breuleux, Bastien, Lamblin, Pascanu,
  Desjardins, Turian, Warde-Farley, and Bengio]{theano1}
Bergstra, James, Breuleux, Olivier, Bastien, Fr{\'{e}}d{\'{e}}ric, Lamblin,
  Pascal, Pascanu, Razvan, Desjardins, Guillaume, Turian, Joseph, Warde-Farley,
  David, and Bengio, Yoshua.
\newblock Theano: A {CPU} and {GPU} math expression compiler.
\newblock In \emph{Proceedings of the Python for Scientific Computing
  Conference ({SciPy})}, June 2010.
\newblock Oral Presentation.

\bibitem[Calders \& Verwer(2010)Calders and Verwer]{fair_naive_bayes}
Calders, Toon and Verwer, Sicco.
\newblock Three {Naive Bayes} approaches for discrimination-free
  classification.
\newblock \emph{Data Mining and Knowledge Discovery}, 21\penalty0 (2):\penalty0
  277--292, 2010.
\newblock ISSN 1384-5810.
\newblock \doi{10.1007/s10618-010-0190-x}.
\newblock URL \url{http://dx.doi.org/10.1007/s10618-010-0190-x}.

\bibitem[Denton et~al.(2015)Denton, Chintala, Szlam, and Fergus]{pyramid_adv}
Denton, Emily, Chintala, Soumith, Szlam, Arthur, and Fergus, Rob.
\newblock Deep generative image models using a {Laplacian Pyramid} of
  adversarial networks.
\newblock \emph{arXiv preprint arXiv:1506.05751}, 2015.

\bibitem[Erickson et~al.()Erickson, {Compiano}, and {Shin}]{neural_net_privacy}
Erickson, {Zackory}, {Compiano}, {Jared}, and {Shin}, {Richard}.
\newblock Neural networks for improving wearable device security.

\bibitem[Ganin et~al.(2015)Ganin, Ustinova, Ajakan, Germain, Larochelle,
  Laviolette, Marchand, and Lempitsky]{domain_adapt_adv}
Ganin, Yaroslav, Ustinova, Evgeniya, Ajakan, Hana, Germain, Pascal, Larochelle,
  Hugo, Laviolette, Fran{\c{c}}ois, Marchand, Mario, and Lempitsky, Victor~S.
\newblock Domain-adversarial training of neural networks.
\newblock \emph{CoRR}, abs/1505.07818, 2015.
\newblock URL \url{http://arxiv.org/abs/1505.07818}.

\bibitem[Gauthier(2015)]{gen_adv_cond1}
Gauthier, J.
\newblock Conditional generative adversarial nets for convolutional face
  generation, 2015.

\bibitem[Goodfellow(2014)]{distinguishability_criteria}
Goodfellow, Ian~J.
\newblock On distinguishability criteria for estimating generative models.
\newblock \emph{arXiv preprint arXiv:1412.6515}, 2014.

\bibitem[Goodfellow et~al.(2014)Goodfellow, Pouget{-}Abadie, Mirza, Xu,
  Warde{-}Farley, Ozair, Courville, and Bengio]{gen_adv}
Goodfellow, Ian~J., Pouget{-}Abadie, Jean, Mirza, Mehdi, Xu, Bing,
  Warde{-}Farley, David, Ozair, Sherjil, Courville, Aaron~C., and Bengio,
  Yoshua.
\newblock Generative adversarial nets.
\newblock In \emph{Advances in Neural Information Processing Systems 27: Annual
  Conference on Neural Information Processing Systems 2014, December 8-13 2014,
  Montreal, Quebec, Canada}, pp.\  2672--2680, 2014.
\newblock URL
  \url{http://papers.nips.cc/paper/5423-generative-adversarial-nets}.

\bibitem[Kamiran \& Calders(2009)Kamiran and Calders]{fair_relabelling}
Kamiran, Faisal and Calders, Toon.
\newblock Classification with no discrimination by preferential sampling.
\newblock In \emph{In 2nd International Conference on Computer, Control and
  Communication}, pp.\  1--6, 2009.

\bibitem[Kamishima et~al.(2011)Kamishima, Akaho, and Sakuma]{fair_reg_log_reg}
Kamishima, T., Akaho, S., and Sakuma, J.
\newblock Fairness-aware learning through regularization approach.
\newblock In \emph{Data Mining Workshops (ICDMW), 2011 IEEE 11th International
  Conference on}, pp.\  643--650, Dec 2011.
\newblock \doi{10.1109/ICDMW.2011.83}.

\bibitem[Kifer et~al.(2004)Kifer, Ben-David, and Gehrke]{detecting_change}
Kifer, Daniel, Ben-David, Shai, and Gehrke, Johannes.
\newblock Detecting change in data streams.
\newblock In \emph{Proceedings of the Thirtieth International Conference on
  Very Large Data Bases - Volume 30}, VLDB '04, pp.\  180--191. VLDB Endowment,
  2004.
\newblock ISBN 0-12-088469-0.
\newblock URL \url{http://dl.acm.org/citation.cfm?id=1316689.1316707}.

\bibitem[Kingma \& Ba(2014)Kingma and Ba]{Adam}
Kingma, Diederik and Ba, Jimmy.
\newblock Adam: A method for stochastic optimization.
\newblock \emph{arXiv preprint arXiv:1412.6980}, 2014.

\bibitem[Korayem et~al.(2014)Korayem, Templeman, Chen, Crandall, and
  Kapadia]{screenavoider}
Korayem, Mohammed, Templeman, Robert, Chen, Dennis, Crandall, David~J., and
  Kapadia, Apu.
\newblock {S}creen{A}voider: protecting computer screens from ubiquitous
  cameras.
\newblock \emph{CoRR}, abs/1412.0008, 2014.
\newblock URL \url{http://arxiv.org/abs/1412.0008}.

\bibitem[Lichman(2013)]{UCI_rep}
Lichman, M.
\newblock {UCI} machine learning repository, 2013.
\newblock URL \url{http://archive.ics.uci.edu/ml}.

\bibitem[{Louizos} et~al.(2015){Louizos}, {Swersky}, {Li}, {Welling}, and
  {Zemel}]{variational_fair_autoencoder}
{Louizos}, C., {Swersky}, K., {Li}, Y., {Welling}, M., and {Zemel}, R.
\newblock {The variational fair auto encoder}.
\newblock \emph{ArXiv e-prints}, November 2015.

\bibitem[Nair \& Hinton(2010)Nair and Hinton]{ReLU}
Nair, Vinod and Hinton, Geoffrey~E.
\newblock Rectified linear units improve restricted boltzmann machines.
\newblock In \emph{Proceedings of the 27th International Conference on Machine
  Learning (ICML-10), June 21-24, 2010, Haifa, Israel}, pp.\  807--814, 2010.
\newblock URL \url{http://www.icml2010.org/papers/432.pdf}.

\bibitem[Rudy \& Taylor(2014)Rudy and Taylor]{gen_adv_cond2}
Rudy, Jan and Taylor, Graham.
\newblock Generative class-conditional autoencoders.
\newblock \emph{arXiv preprint arXiv:1412.7009}, 2014.

\bibitem[Schmidhuber(1991)]{factorial_codes}
Schmidhuber, J.
\newblock {Learning factorial codes by predictability minimization}.
\newblock Technical Report CU-CS-565-91, Dept. of Comp. Sci., University of
  Colorado at Boulder, December 1991.

\bibitem[Schmidhuber et~al.(1996)Schmidhuber, Eldracher, and
  Foltin]{semilinear_predictability_minimization}
Schmidhuber, J, Eldracher, M, and Foltin, B.
\newblock Semilinear predictability minimization produces well-known feature
  detectors.
\newblock \emph{Neural Computation}, 8\penalty0 (4):\penalty0 773--786, May
  1996.
\newblock ISSN 0899-7667.
\newblock \doi{10.1162/neco.1996.8.4.773}.

\bibitem[Schraudolph et~al.(1999)Schraudolph, Eldracher, and
  Schmid\-huber]{semilinear_predictability_minimization2}
Schraudolph, Nicol~N., Eldracher, Martin, and Schmid\-huber, J\"urgen.
\newblock \href{http://nic.schraudolph.org/pubs/SchEldSch99.pdf}{ Processing
  images by semi-linear predictability minimization}.
\newblock \emph{Network: Computation in Neural Systems}, 10\penalty0
  (2):\penalty0 133--169, 1999.

\bibitem[Zafar et~al.(2015)Zafar, Valera, Gomez{-}Rodriguez, and
  Gummadi]{fairness_constraint}
Zafar, Muhammad~Bilal, Valera, Isabel, Gomez{-}Rodriguez, Manuel, and Gummadi,
  Krishna~P.
\newblock Fairness constraints: {A} mechanism for fair classification.
\newblock \emph{CoRR}, abs/1507.05259, 2015.
\newblock URL \url{http://arxiv.org/abs/1507.05259}.

\bibitem[Zemel et~al.(2013)Zemel, Wu, Swersky, Pitassi, and Dwork]{LFR}
Zemel, Rich, Wu, Yu, Swersky, Kevin, Pitassi, Toni, and Dwork, Cynthia.
\newblock Learning fair representations.
\newblock In Dasgupta, Sanjoy and Mcallester, David (eds.), \emph{Proceedings
  of the 30th International Conference on Machine Learning (ICML-13)},
  volume~28, pp.\  325--333. JMLR Workshop and Conference Proceedings, May
  2013.
\newblock URL \url{http://jmlr.org/proceedings/papers/v28/zemel13.pdf}.

\end{thebibliography}
\bibliographystyle{iclr2016_conference}

\newpage
\appendix 
\section{The $\mathcal{H}$-Divergence} \label{App:AppendixA}
In this appendix we describe the notion of a $\mathcal{H}$-divergence, as developed in \cite{detecting_change} and \cite{rep_dom}. The $\mathcal{H}$-divergence is a way of measuring the difference between two distributions using classifiers. 

\begin{definition}
A \emph{hypothesis} for a random variable $X$ is a mapping $\eta: \mathcal{X} \to \{0,1 \}$.
\end{definition}
\begin{definition}
A \emph{hypothesis class} $\mathcal{H}$ for a random variable $X$ is a collection of hypotheses on $X$.
\end{definition}
\begin{definition}
A \emph{symmetrical hypothesis class} $\mathcal{H}$ is a hypothesis class such that for each $\eta \in H$, the inverse hypothesis $\overline{\eta}$ defined $\overline{\eta}(x) = 1 - \eta(x)$ is also in $\mathcal{H}.$
\end{definition}

Now we can define the divergence.

\begin{definition}
\label{H_divergence}
Given two random variables $X_0, X_1$ on a common space $\mathcal{X}$ and a hypothesis class $\mathcal{H}$ on $\mathcal{X}$, the $\mathcal{H}$-\emph{divergence} between $X_0$ and $X_1$ is
\[
d_{\mathcal{H}}(X_0, X_1) = \sup_{\eta \in \mathcal{H}} 2 \left |  \Expt \eta(X_0) - \Expt \eta(X_1) \right |.
\]
\end{definition}

In case $\mathcal{H}$ is symmetric, \cite{theory_different_domains} show that $d_{\mathcal{H}}(X_0, X_1)$ can be approximated empirically.

\begin{definition}
\label{empirical_H_divergence}
Let  $\mathcal{H}$ be a symmetrical hypothesis class on random variables $X_0, X_1$ on a common space $\mathcal{X}$. Now given \emph{i.i.d} samples $A = \{a_1, \dots, a_n \}$ from $X_0$ and \emph{i.i.d} samples $B = \{b_1, \dots, d_m \}$ from $X_1$, we define the \emph{empirical} $\mathcal{H}$-divergence between $X_0, X_1$ to be

\[
\hat{d}_{\mathcal{H}} (X_0, X_1) = 
 \sup_{\eta \in \mathcal{H}} 2
\left [
\frac{1}{m}\suml{x \in A}{} \eta(x) - \frac{1}{n} \suml{x \in B}{} \eta(x)
\right ],
\]

\end{definition}
The nature of this empirical approximation is shown in the following probabilistic bound from \cite{theory_different_domains}.
\begin{lemma}
Let  $ \mathcal{H} $ be a symmetrical hypothesis class with VC dimension $d$ on random variables $X_0, X_1$ on a common space $\mathcal{X}$. Now given \emph{i.i.d} samples $A = \{a_1, \dots, a_m \}$ from $X_0$ and \emph{i.i.d} samples $B = \{b_1, \dots, d_m \}$ from $X_1$, we have that, for any $\delta \in (0,1)$, with probability at least $1-\delta$,
\[
d_{\mathcal{H}}(X_0, X_1) \le \hat{d}_{\mathcal{H}} (A,B) + \sqrt{ \frac{d \log(2m) + \log(\frac{2}{\delta})}{m}}.
\]

\end{lemma}

So we can see that minimizing the capability of an adversary to tell the difference between two distributions relates to minimizing an $\mathcal{H}$-divergence. The empirical divergence can be straightforwardly related to the discrimination of a classifier.

\begin{lemma}
\label{disc_divergence_lemma}
Let $X$, $S$, $Y$ and $R$ be as described in Section \ref{censored_representations}. Let $\eta: \mathcal{R} \to \mathcal{Y}$ be our classifier and $y_{disc}$ the discrimination as described in Section \ref{fairness_and_disc}.
Then
\[
  y_{disc} \le \frac{1}{2} \hat{d}_{\mathcal{H}}(\hat{R}_0,\hat{R}_1),
\]
where $\mathcal{H}$ is a symmetric hypothesis class on $R$ including $\eta$ and $\hat{R}_0, \hat{R}_1$ are the representations of the empirical samples.
\end{lemma}

\begin{proof}
Recall that
\[
y_{disc} = \left | \frac{\suml{r \in \hat{R_0}}{} \eta(r)}{|\hat{R_0}|} -  \frac{\suml{r \in \hat{R_1}}{} \eta(r)}{|\hat{R_1}|} \right |,
\]
where $m = |\hat{R_0}| = |\hat{R_1}|$. We simply observe that
\begin{align*}
 \left | \frac{\suml{r \in \hat{R_0}}{} \eta(r)}{|\hat{R_0}|} -  \frac{\suml{r \in \hat{R_1}}{} \eta(r)}{|\hat{R_1}|} \right | &\le \sup _{\eta^* \in \mathcal{H}} \left | \frac{\suml{r \in \hat{R_0}}{} \eta^*(r)}{|\hat{R_0}|} -  \frac{\suml{r \in \hat{R_1}}{} \eta^*(r)}{|\hat{R_1}|} \right | \\
&=  \sup _{\eta^* \in \mathcal{H}} \left [  \frac{\suml{r \in \hat{R_0}}{} \eta^*(r)}{|\hat{R_0}|} -  \frac{\suml{r \in \hat{R_1}}{} \eta^*(r)}{|\hat{R_1}|} \right ] \\
&= \frac{1}{2} \hat{D}_{\mathcal{H}}(\hat{R_0}, \hat{R_1})
\end{align*}
where for the inequality we use the fact that $\eta \in \mathcal{H}$ and for the first equality we use the fact that $\mathcal{H}$ is a symmetric hypothesis class.

\end{proof}

\end{document}